\documentclass{article}
\usepackage{algorithm}
\usepackage[noend]{algorithmic}
\usepackage{amsmath}
\usepackage{amssymb}
\usepackage{amsthm}
\usepackage{bbm}
\usepackage{bm}
\usepackage{color}
\usepackage{dsfont}
\usepackage{enumerate}
\usepackage{graphicx}
\usepackage{hyperref}
\usepackage[accepted]{icml2015}
\usepackage{natbib}
\usepackage{subfigure}
\usepackage{times}

\usepackage[capitalize]{cleveref}

\newcommand{\commentout}[1]{}
\newcommand{\junk}[1]{}

\newtheorem{theorem}{Theorem}

\newtheorem{lemma}{Lemma}

\newtheorem{assumption}{Assumption}

\newcommand{\cE}{\mathcal{E}}
\newcommand{\ccE}{\overline{\cE}}

\newcommand{\cH}{\mathcal{H}}

\newcommand{\eps}{\varepsilon}

\newcommand{\abs}[1]{\left|#1\right|}

\newcommand{\condEE}[2]{\mathbb{E} \left[#1 \,\middle|\, #2\right]}

\newcommand{\EE}[1]{\mathbb{E} \left[#1\right]}
\newcommand{\EEt}[1]{\mathbb{E}_t \left[#1\right]}

\newcommand{\I}[1]{\mathds{1} \! \left\{#1\right\}}

\newcommand{\kl}[2]{D_\mathrm{KL}(#1 \,\|\, #2)}

\newcommand{\rnd}[1]{\mathbf{#1}}
\newcommand{\set}[1]{\left\{#1\right\}}

\newcommand{\ud}{\, \mathrm{d}}

\DeclareMathOperator*{\argmax}{arg\,max\,}
\DeclareMathOperator*{\argmin}{arg\,min\,}
\mathchardef\mhyphen="2D

\newcommand{\cascadeucb}{{\tt CascadeUCB1}}
\newcommand{\cascadeklucb}{{\tt CascadeKL\mhyphen UCB}}
\newcommand{\combucb}{{\tt CombUCB1}}
\newcommand{\klucb}{{\tt KL\mhyphen UCB}}
\newcommand{\rankedklucb}{{\tt RankedKL\mhyphen UCB}}
\newcommand{\ucb}{{\tt UCB1}}

\begin{document}

\icmltitlerunning{Cascading Bandits: Learning to Rank in the Cascade Model}

\twocolumn[
\icmltitle{Cascading Bandits: Learning to Rank in the Cascade Model}
\icmlauthor{Branislav Kveton}{kveton@adobe.com}
\icmladdress{Adobe Research, San Jose, CA}
\icmlauthor{Csaba Szepesv\'ari}{szepesva@cs.ualberta.ca}
\icmladdress{Department of Computing Science, University of Alberta}
\icmlauthor{Zheng Wen}{zhengwen@yahoo-inc.com}
\icmladdress{Yahoo Labs, Sunnyvale, CA}
\icmlauthor{Azin Ashkan}{azin.ashkan@technicolor.com}
\icmladdress{Technicolor Research, Los Altos, CA}
\vskip 0.3in]

\begin{abstract}
A search engine usually outputs a list of $K$ web pages. The user examines this list, from the first web page to the last, and chooses the first attractive page. This model of user behavior is known as the cascade model. In this paper, we propose \emph{cascading bandits}, a learning variant of the cascade model where the objective is to identify $K$ most attractive items. We formulate our problem as a stochastic combinatorial partial monitoring problem. We propose two algorithms for solving it, $\cascadeucb$ and $\cascadeklucb$. We also prove gap-dependent upper bounds on the regret of these algorithms and derive a lower bound on the regret in cascading bandits. The lower bound matches the upper bound of $\cascadeklucb$ up to a logarithmic factor. We experiment with our algorithms on several problems. The algorithms perform surprisingly well even when our modeling assumptions are violated.
\end{abstract}

%!TEX root = Paper.tex

\section{Introduction}
\label{sec:introduction}

The \emph{cascade model} is a popular model of user behavior in web search \cite{craswell08experimental}. In this model, the user is recommended a list of $K$ items, such as web pages. The user \emph{examines} the recommended list from the first item to the last, and selects the first \emph{attractive} item. In web search, this is manifested as a click. The items before the first attractive item are \emph{not attractive}, because the user examines these items but does not click on them. The items after the first attractive item are \emph{unobserved}, because the user never examines these items. The optimal list, the list of $K$ items that maximizes the probability that the user finds an attractive item, are $K$ most attractive items. The cascade model is simple but effective in explaining the so-called position bias in historical click data \cite{craswell08experimental}. Therefore, it is a reasonable model of user behavior.

In this paper, we propose an online learning variant of the cascade model, which we refer to as \emph{cascading bandits}. In this model, the learning agent does not know the attraction probabilities of items. At time $t$, the agent recommends to the user a list of $K$ items out of $L$ items and then observes the index of the item that the user clicks. If the user clicks on an item, the agent receives a reward of one. The goal of the agent is to maximize its total reward, or equivalently to minimize its cumulative regret with respect to the list of $K$ most attractive items. Our learning problem can be viewed as a bandit problem where the reward of the agent is a part of its feedback. But the feedback is richer than the reward. Specifically, the agent knows that the items before the first attractive item are not attractive.

We make five contributions. First, we formulate a learning variant of the cascade model as a stochastic combinatorial partial monitoring problem. Second, we propose two algorithms for solving it, $\cascadeucb$ and $\cascadeklucb$. $\cascadeucb$ is motivated by $\combucb$, a computationally and sample efficient algorithm for stochastic combinatorial semi-bandits \cite{gai12combinatorial,kveton15tight}. $\cascadeklucb$ is motivated by $\klucb$  and we expect it to perform better when the attraction probabilities of items are low \cite{garivier11klucb}. This setting is common in the problems of our interest, such as web search. Third, we prove gap-dependent upper bounds on the regret of our algorithms. Fourth, we derive a lower bound on the regret in cascading bandits. This bound matches the upper bound of $\cascadeklucb$ up to a logarithmic factor. Finally, we experiment with our algorithms on several problems. They perform well even when our modeling assumptions are not satisfied.

Our paper is organized as follows. In \cref{sec:background}, we review the cascade model. In \cref{sec:cascading bandits}, we introduce our learning problem and propose two UCB-like algorithms for solving it. In \cref{sec:analysis}, we derive gap-dependent upper bounds on the regret of $\cascadeucb$ and $\cascadeklucb$. In addition, we prove a lower bound and discuss how it relates to our upper bounds. We experiment with our learning algorithms in \cref{sec:experiments}. In \cref{sec:related work}, we review related work. We conclude in \cref{sec:conclusions}.

%!TEX root = Paper.tex

\section{Background}
\label{sec:background}

Web pages in a search engine can be ranked automatically by fitting a model of user behavior in web search from historical click data \cite{radlinski05query,agichtein06improving}. The user is typically assumed to scan a list of $K$ web pages $A = (a_1, \dots, a_K)$, which we call \emph{items}. The items belong to some \emph{ground set} $E = \set{1, \dots, L}$, such as the set of all web pages. Many models of user behavior in web search exist \cite{becker07modeling,craswell08experimental,richardson07predicting}. Each of them explains the clicks of the user differently. We focus on the cascade model.

The \emph{cascade model} is a popular model of user behavior in web search \cite{craswell08experimental}. In this model, the user scans a list of $K$ items $A = (a_1, \dots, a_K) \in \Pi_K(E)$ from the first item $a_1$ to the last $a_K$, where $\Pi_K(E)$ is the set of all \emph{$K$-permutations} of set $E$. The model is parameterized by \emph{attraction probabilities} $\bar{w} \in [0, 1]^E$. After the user examines item $a_k$, the item attracts the user with probability $\bar{w}(a_k)$, \emph{independently} of the other items. If the user is attracted by item $a_k$, the user clicks on it and does not examine the remaining items. If the user is not attracted by item $a_k$, the user examines item $a_{k + 1}$. It is easy to see that the probability that item $a_k$ is examined is $\prod_{i = 1}^{k - 1} (1 - \bar{w}(a_i))$, and that the probability that at least one item in $A$ is attractive is $1 - \prod_{i = 1}^{K} (1 - \bar{w}({a_i}))$. This objective is maximized by $K$ most attractive items.

The cascade model assumes that the user clicks on at most one item. In practice, the user may click on multiple items. The cascade model cannot explain this pattern. Therefore, the model was extended in several directions, for instance to take into account \emph{multiple clicks} and the \emph{persistence} of users \cite{chapelle09dynamic,guo09click,guo09efficient}. The extended models explain click data better than the cascade model. Nevertheless, the cascade model is still very attractive, because it is simpler and can be reasonably fit to click data. Therefore, as a first step towards understanding more complex models, we study an online variant of the cascade model in this work.

%!TEX root = Paper.tex

\section{Cascading Bandits}
\label{sec:cascading bandits}

We propose a learning variant of the cascade model (\cref{sec:setting}) and two computationally-efficient algorithms for solving it (\cref{sec:algorithms}). To simplify exposition, all random variables are written in bold.

\subsection{Setting}
\label{sec:setting}

We refer to our learning problem as a \emph{generalized cascading bandit}. Formally, we represent the problem by a tuple $B = (E, P, K)$, where $E = \set{1, \dots, L}$ is a \emph{ground set} of $L$ items, $P$ is a probability distribution over a unit hypercube $\set{0, 1}^E$, and $K \leq L$ is the number of recommended items. We call the bandit \emph{generalized} because the form of the distribution $P$ has not been specified yet.

Let $(\rnd{w}_t)_{t = 1}^n$ be an i.i.d. sequence of $n$ \emph{weights} drawn from $P$, where $\rnd{w}_t \in \set{0, 1}^E$ and $\rnd{w}_t(e)$ is the preference of the user for item $e$ at time $t$. That is, $\rnd{w}_t(e) = 1$ if and only if item $e$ attracts the user at time $t$. The learning agent interacts with our problem as follows. At time $t$, the agent recommends a list of $K$ items $\rnd{A}_t = (\rnd{a}^t_1, \dots, \rnd{a}^t_K) \in \Pi_K(E)$. The list is computed from the observations of the agent up to time $t$. The user examines the list, from the first item $\rnd{a}^t_1$ to the last $\rnd{a}^t_K$, and clicks on the first attractive item. If the user is not attracted by any item, the user does not click on any item. Then time increases to $t + 1$.

The reward of the agent at time $t$ can be written in several forms. For instance, as $\max_k \rnd{w}_t(\rnd{a}^t_k)$, at least one item in list $\rnd{A}_t$ is attractive; or as $f(\rnd{A}_t, \rnd{w}_t)$, where:
\begin{align*}
  f(A, w) = 1 - \prod_{k = 1}^K (1 - w(a_k))\,,
\end{align*}
$A = (a_1, \dots, a_K) \in \Pi_K(E)$, and $w \in \set{0, 1}^E$. This later algebraic form is particularly useful in our proofs.

The agent at time $t$ receives feedback:
\begin{align*}
  \rnd{C}_t = \argmin \set{1 \leq k \leq K: \rnd{w}_t(\rnd{a}^t_k) = 1}\,,
\end{align*}
where we assume that $\argmin \emptyset = \infty$. The feedback $\rnd{C}_t$ is the click of the user. If $\rnd{C}_t \leq K$, the user clicks on item $\rnd{C}_t$. If $\rnd{C}_t = \infty$, the user does not click on any item. Since the user clicks on the first attractive item in the list, we can determine the observed weights of all recommended items at time $t$ from $\rnd{C}_t$. In particular, note that:
\begin{align}
  \rnd{w}_t(\rnd{a}^t_k) = \I{\rnd{C}_t = k} \quad k = 1, \dots, \min \set{\rnd{C}_t, K}\,.
  \label{eq:click}
\end{align}
We say that item $e$ is \emph{observed} at time $t$ if $e = \rnd{a}^t_k$ for some $1 \leq k \leq \min \set{\rnd{C}_t, K}$.

In the cascade model (\cref{sec:background}), the weights of the items in the ground set $E$ are distributed independently. We also make this assumption.

\begin{assumption}
\label{ass:independence} The weights $w$ are distributed as:
\begin{align*}
  P(w) = \prod_{e \in E} P_e(w(e))\,,
\end{align*}
where $P_e$ is a Bernoulli distribution with mean $\bar{w}(e)$.
\end{assumption}

Under this assumption, we refer to our learning problem as a \emph{cascading bandit}. In this new problem, the weight of any item at time $t$ is drawn independently of the weights of the other items at that, or any other, time. This assumption has profound consequences and leads to a particularly efficient learning algorithm in \cref{sec:algorithms}. More specifically, under our assumption, the expected reward for list $A \in \Pi_K(E)$, the probability that at least one item in $A$ is attractive, can be expressed as $\EE{f(A, \rnd{w})} = f(A, \bar{w})$, and depends only on the attraction probabilities of individual items in $A$.

The agent's policy is evaluated by its \emph{expected cumulative regret}:
\begin{align*}
  R(n) = \EE{\sum_{t = 1}^n R(\rnd{A}_t, \rnd{w}_t)}\,,
\end{align*}
where $R(\rnd{A}_t, \rnd{w}_t) = f(A^\ast, \rnd{w}_t) - f(\rnd{A}_t, \rnd{w}_t)$ is the \emph{instantaneous stochastic regret} of the agent at time $t$ and:
\begin{align*}
  A^\ast = \argmax_{A \in \Pi_K(E)} f(A, \bar{w})
\end{align*}
is the \emph{optimal list} of items, the list that maximized the reward at any time $t$. Since $f$ is invariant to the permutation of $A$, there exist at least $K!$ optimal lists. For simplicity of exposition, we assume that the optimal solution, as a set, is unique.

\subsection{Algorithms}
\label{sec:algorithms}

We propose two algorithms for solving cascading bandits, $\cascadeucb$ and $\cascadeklucb$. $\cascadeucb$ is motivated by $\ucb$ \cite{auer02finitetime} and $\cascadeklucb$ is motivated by $\klucb$ \cite{garivier11klucb}.

The pseudocode of both algorithms is in \cref{alg:ucb}. The algorithms are similar and differ only in how they estimate the \emph{upper confidence bound (UCB)} $\rnd{U}_t(e)$ on the attraction probability of item $e$ at time $t$. After that, they recommend a list of $K$ items with largest UCBs:
\begin{align}
  \rnd{A}_t = \argmax_{A \in \Pi_K(E)} f(A, \rnd{U}_t)\,.
\end{align}
Note that $\rnd{A}_t$ is determined only up to a permutation of the items in it. The payoff is not affected by this ordering. But the observations are. For now, we leave the order of items unspecified and return to it later in our discussions. After the user provides feedback $\rnd{C}_t$, the algorithms update their estimates of the attraction probabilities $\bar{w}(e)$ based on \eqref{eq:click}, for all $e = \rnd{a}^t_k$ where $k \leq \rnd{C}_t$.

\begin{algorithm}[t]
  \caption{UCB-like algorithm for cascading bandits.}
  \label{alg:ucb}
  \begin{algorithmic}
    \STATE // Initialization
    \STATE Observe $\rnd{w}_0 \sim P$
    \STATE $\forall e \in E: \rnd{T}_0(e) \gets 1$
    \STATE $\forall e \in E: \hat{\rnd{w}}_1(e) \gets \rnd{w}_0(e)$
    \STATE
    \FORALL{$t = 1, \dots, n$}
      \STATE Compute UCBs $\rnd{U}_t(e)$ (\cref{sec:algorithms})
      \STATE
      \STATE // Recommend a list of $K$ items and get feedback
      \STATE Let $\rnd{a}^t_1, \dots, \rnd{a}^t_K$ be $K$ items with largest UCBs
      \STATE $\rnd{A}_t \gets (\rnd{a}^t_1, \dots, \rnd{a}^t_K)$
      \STATE Observe click $\rnd{C}_t \in \set{1, \dots, K, \infty}$
      \STATE
      \STATE // Update statistics
      \STATE $\forall e \in E: \rnd{T}_t(e) \gets \rnd{T}_{t - 1}(e)$
      \FORALL{$k = 1, \dots, \min \set{\rnd{C}_t, K}$}
        \STATE $e \gets \rnd{a}^t_k$
        \STATE $\rnd{T}_t(e) \gets \rnd{T}_t(e) + 1$
        \STATE $\displaystyle \hat{\rnd{w}}_{\rnd{T}_t(e)}(e) \gets
        \frac{\rnd{T}_{t - 1}(e) \hat{\rnd{w}}_{\rnd{T}_{t - 1}(e)}(e) + \I{\rnd{C}_t = k}}{\rnd{T}_t(e)}$
      \ENDFOR
    \ENDFOR
  \end{algorithmic}
\end{algorithm}

The UCBs are computed as follows. In $\cascadeucb$, the UCB on the attraction probability of item $e$ at time $t$ is:
\begin{align*}
  \rnd{U}_t(e) = \hat{\rnd{w}}_{\rnd{T}_{t - 1}(e)}(e) + c_{t - 1, \rnd{T}_{t - 1}(e)}\,,
\end{align*}
where $\hat{\rnd{w}}_s(e)$ is the average of $s$ observed weights of item $e$, $\rnd{T}_t(e)$ is the number of times that item $e$ is observed in $t$ steps, and:
\begin{align*}
  c_{t, s} = \sqrt{(1.5 \log t) / s}
\end{align*}
is the radius of a confidence interval around $\hat{\rnd{w}}_s(e)$ after $t$ steps such that $\bar{w}(e) \in [\hat{\rnd{w}}_s(e) - c_{t, s}, \hat{\rnd{w}}_s(e) + c_{t, s}]$ holds with high probability. In $\cascadeklucb$, the UCB on the attraction probability of item $e$ at time $t$ is:
\begin{align*}
  & \rnd{U}_t(e) = \max\{q \in [\hat{\rnd{w}}_{\rnd{T}_{t - 1}(e)}(e), 1]: \\
  & \quad \rnd{T}_{t - 1}(e) \kl{\hat{\rnd{w}}_{\rnd{T}_{t - 1}(e)}(e)}{q} \leq \log t + 3 \log \log t\}\,,
\end{align*}
where $\kl{p}{q}$ is the \emph{Kullback-Leibler (KL) divergence} between two Bernoulli random variables with means $p$ and $q$. Since $\kl{p}{q}$ is an increasing function of $q$ for $q \geq p$, the above UCB can be computed efficiently.

\subsection{Initialization}
\label{sec:initialization}

Both algorithms are initialized by one sample $\rnd{w}_0$ from $P$. Such a sample can be generated in $O(L)$ steps, by recommending each item once as the first item in the list.

%!TEX root = Paper.tex

\section{Analysis}
\label{sec:analysis}

Our analysis exploits the fact that our reward and feedback models are closely connected. More specifically, we show in \cref{sec:regret decomposition} that the learning algorithm can suffer regret only if it recommends suboptimal items that are observed. Based on this result, we prove upper bounds on the $n$-step regret of $\cascadeucb$ and $\cascadeklucb$ (\cref{sec:upper bounds}). We prove a lower bound on the regret in cascading bandits in \cref{sec:lower bound}. We discuss our results in \cref{sec:discussion}.

\subsection{Regret Decomposition}
\label{sec:regret decomposition}

Without loss of generality, we assume that the items in the ground set $E$ are sorted in decreasing order of their attraction probabilities, $\bar{w}(1) \geq \ldots \geq \bar{w}(L)$. In this setting, the \emph{optimal solution} is $A^\ast = (1, \dots, K)$, and contains the first $K$ items in $E$. We say that item $e$ is \emph{optimal} if $1 \leq e \leq K$. Similarly, we say that item $e$ is \emph{suboptimal} if $K < e \leq L$. The \emph{gap} between the attraction probabilities of suboptimal item $e$ and optimal item $e^\ast$:
\begin{align}
  \Delta_{e, e^\ast} = \bar{w}(e^\ast) - \bar{w}(e)
  \label{eq:gap}
\end{align}
measures the hardness of discriminating the items. Whenever convenient, we view an ordered list of items as the set of items on that list.

Our main technical lemma is below. The lemma says that the expected value of the difference of the products of random variables can be written in a particularly useful form.

\begin{lemma}
\label{lem:modular decomposition} Let $A = (a_1, \dots, a_K)$ and $B = (b_1, \dots, b_K)$ be any two lists of $K$ items from $\Pi_K(E)$ such that $a_i = b_j$ only if $i = j$. Let $\rnd{w} \sim P$ in \cref{ass:independence}. Then:
\begin{align*}
  & \EE{\prod_{k = 1}^K \rnd{w}(a_k) - \prod_{k = 1}^K \rnd{w}(b_k)} =
  \sum_{k = 1}^K \EE{\prod_{i = 1}^{k - 1} \rnd{w}(a_i)} \times {} \\
  & \quad \EE{\rnd{w}(a_k) - \rnd{w}(b_k)} \left(\prod_{j = k + 1}^K \EE{\rnd{w}(b_j)}\right)\,.
\end{align*}
\end{lemma}
\begin{proof}
The claim is proved in \cref{sec:lemmas}.
\end{proof}

Let:
\begin{align}
  \cH_t = (\rnd{A}_1, \rnd{C}_1, \dots, \rnd{A}_{t - 1}, \rnd{C}_{t - 1}, \rnd{A}_t)
  \label{eq:history}
\end{align}
be the \emph{history} of the learning agent up to choosing $\rnd{A}_t$, the first $t - 1$ observations and $t$ actions. Let $\EEt{\cdot} = \condEE{\cdot}{\cH_t}$ be the conditional expectation given history $\cH_t$. We bound $\EEt{R(\rnd{A}_t, \rnd{w}_t)}$, the expected regret conditioned on history $\cH_t$, as follows.

\begin{theorem}
\label{thm:regret decomposition} For any item $e$ and optimal item $e^\ast$, let:
\begin{align}
  G_{e, e^\ast, t} & = \{\exists 1 \leq k \leq K \text{ s.t. } \rnd{a}^t_k = e, \ \pi_t(k) = e^\ast,
  \label{eq:counted event} \\
  & \phantom{{} = \{} \rnd{w}_t(\rnd{a}^t_1) = \ldots = \rnd{w}_t(\rnd{a}^t_{k - 1}) = 0\} \nonumber
\end{align}
be the event that item $e$ is chosen instead of item $e^\ast$ at time $t$, and that item $e$ is observed. Then there exists a permutation $\pi_t$ of optimal items $\set{1, \dots, K}$, which is a deterministic function of $\cH_t$, such that $\rnd{U}_t(\rnd{a}^t_k) \geq \rnd{U}_t(\pi_t(k))$ for all $k$. Moreover:
\begin{align*}
  \EEt{R(\rnd{A}_t, \rnd{w}_t)} & \leq
  \sum_{e = K + 1}^L \sum_{e^\ast = 1}^K \Delta_{e, e^\ast} \EEt{\I{G_{e, e^\ast, t}}} \\
  \EEt{R(\rnd{A}_t, \rnd{w}_t)} & \geq
  \alpha \sum_{e = K + 1}^L \sum_{e^\ast = 1}^K \Delta_{e, e^\ast} \,\EEt{\I{G_{e, e^\ast, t}}}\,,
\end{align*}
where $\alpha = (1 - \bar{w}(1))^{K - 1}$ and $\bar{w}(1)$ is the attraction probability of the most attractive item.
\end{theorem}
\begin{proof}
We define $\pi_t$ as follows. For any $k$, if the $k$-th item in $\rnd{A}_t$ is optimal, we place this item at position $k$, $\pi_t(k) = \rnd{a}^t_k$. The remaining optimal items are positioned arbitrarily. Since $A^\ast$ is optimal with respect to $\bar{w}$, $\bar{w}(\rnd{a}^t_k) \leq \bar{w}(\pi_t(k))$ for all $k$. Similarly, since $\rnd{A}_t$ is optimal with respect to $\rnd{U}_t$, $\rnd{U}_t(\rnd{a}^t_k) \geq \rnd{U}_t(\pi_t(k))$ for all $k$. Therefore, $\pi_t$ is the desired permutation.

The permutation $\pi_t$ reorders the optimal items in a convenient way. Since time $t$ is fixed, let $\rnd{a}^\ast_k = \pi_t(k)$. Then:
\begin{align*}
  & \EEt{R(\rnd{A}_t, \rnd{w}_t)} = \\
  & \quad \EEt{\prod_{k = 1}^K (1 - \rnd{w}_t(\rnd{a}^t_k)) - \prod_{k = 1}^K (1 - \rnd{w}_t(\rnd{a}^\ast_k))}\,.
\end{align*}
Now we exploit the fact that the entries of $\rnd{w}_t$ are independent of each other given $\cH_t$. By \cref{lem:modular decomposition}, we can rewrite the right-hand side of the above equation as:
\begin{align*}
  & \sum_{k = 1}^K \EEt{\prod_{i = 1}^{k - 1} (1 - \rnd{w}_t(\rnd{a}^t_i))}
  \EEt{\rnd{w}_t(\rnd{a}^\ast_k) - \rnd{w}_t(\rnd{a}^t_k)} \times {} \\
  & \quad \left(\prod_{j = k + 1}^K \EEt{1 - \rnd{w}_t(\rnd{a}^\ast_j)}\right)\,.
\end{align*}
Note that $\EEt{\rnd{w}_t(\rnd{a}^\ast_k) - \rnd{w}_t(\rnd{a}^t_k)} = \Delta_{\rnd{a}^t_k, \rnd{a}^\ast_k}$. Furthermore, $\prod_{i = 1}^{k - 1} (1 - \rnd{w}_t(\rnd{a}^t_i)) = \I{G_{\rnd{a}^t_k, \rnd{a}^\ast_k, t}}$ by conditioning on $\cH_t$. Therefore, we get that $\EEt{R(\rnd{A}_t, \rnd{w}_t)}$ is equal to:
\begin{align*}
  \sum_{k = 1}^K \Delta_{\rnd{a}^t_k, \rnd{a}^\ast_k} \EEt{\I{G_{\rnd{a}^t_k, \rnd{a}^\ast_k, t}}}
  \prod_{j = k + 1}^K \EEt{1 - \rnd{w}_t(\rnd{a}^\ast_j)}\,.
\end{align*}
By definition of $\pi_t$, $\Delta_{\rnd{a}^t_k, \rnd{a}^\ast_k} = 0$ when item $\rnd{a}^t_k$ is optimal. In addition, $1 - \bar{w}(1) \leq \EEt{1 - \rnd{w}_t(\rnd{a}^\ast_j)} \leq 1$ for any optimal $\rnd{a}^\ast_j$. Our upper and lower bounds on $\EEt{R(\rnd{A}_t, \rnd{w}_t)}$ follow from these observations.
\end{proof}

\subsection{Upper Bounds}
\label{sec:upper bounds}

In this section, we derive two upper bounds on the $n$-step regret of $\cascadeucb$ and $\cascadeklucb$.

\begin{theorem}
\label{thm:ucb1} The expected $n$-step regret of $\cascadeucb$ is bounded as:
\begin{align*}
  R(n) \leq
  \sum_{e = K + 1}^L \frac{12}{\Delta_{e, K}} \log n + \frac{\pi^2}{3} L\,.
\end{align*}
\end{theorem}
\begin{proof}
The complete proof is in \cref{sec:proof ucb1}. The proof has four main steps. First, we bound the regret of the event that $\bar{w}(e)$ is outside of the high-probability confidence interval around $\hat{\rnd{w}}_{\rnd{T}_{t - 1}(e)}(e)$ for at least one item $e$. Second, we decompose the regret at time $t$ and apply \cref{thm:regret decomposition} to bound it from above. Third, we bound the number of times that each suboptimal item is chosen in $n$ steps. Fourth, we peel off an extra factor of $K$ in our upper bound based on \citet{kveton14matroid}. Finally, we sum up the regret of all suboptimal items.
\end{proof}

\begin{theorem}
\label{thm:klucb} For any $\eps > 0$, the expected $n$-step regret of $\cascadeklucb$ is bounded as:
\begin{align*}
  R(n)
  & \leq \sum_{e = K + 1}^L
  \frac{(1 + \eps) \Delta_{e, K} (1 + \log(1 / \Delta_{e, K}))}{\kl{\bar{w}(e)}{\bar{w}(K)}} \times {} \\
  & \qquad\qquad\ (\log n + 3 \log \log n) + C\,,
\end{align*}
where $C = K L \frac{C_2(\eps)}{n^{\beta(\eps)}} + 7 K \log \log n$, and the constants $C_2(\eps)$ and $\beta(\eps)$ are defined in \citet{garivier11klucb}. 
\end{theorem}
\begin{proof}
The complete proof is in \cref{sec:proof klucb}. The proof has four main steps. First, we bound the regret of the event that $\bar{w}(e) > \rnd{U}_t(e)$ for at least one optimal item $e$. Second, we decompose the regret at time $t$ and apply \cref{thm:regret decomposition} to bound it from above. Third, we bound the number of times that each suboptimal item is chosen in $n$ steps. Fourth, we derive a new peeling argument for $\klucb$ (\cref{lem:klucb peeling}) and eliminate an extra factor of $K$ in our upper bound. Finally, we sum up the regret of all suboptimal items.
\end{proof}

%!TEX root = Paper.tex

\subsection{Lower Bound}
\label{sec:lower bound}

Our lower bound is derived on the following problem. The ground set contains $L$ items $E = \set{1, \dots, L}$. The distribution $P$ is a product of $L$ Bernoulli distributions $P_e$, each of which is parameterized by:
\begin{align}
  \bar{w}(e) =
  \begin{cases}
    p & e \leq K \\
    p - \Delta & \text{otherwise}\,,
  \end{cases}
  \label{eq:attraction probability}
\end{align}
where $\Delta \in (0, p)$ is the gap between any optimal and suboptimal items. We refer to the resulting bandit problem as $B_\mathrm{LB}(L, K, p, \Delta)$; and parameterize it by $L$, $K$, $p$, and $\Delta$.

Our lower bound holds for consistent algorithms. We say that the algorithm is \emph{consistent} if for any cascading bandit, any suboptimal list $A$, and any $\alpha > 0$, $\EE{\rnd{T}_n(A)} = o(n^\alpha)$, where $\rnd{T}_n(A)$ is the number of times that list $A$ is recommended in $n$ steps. Note that the restriction to the consistent algorithms is without loss of generality. \mbox{The reason is} that any inconsistent algorithm must suffer polynomial regret on some instance of cascading bandits, and therefore cannot achieve logarithmic regret on every instance of our problem, similarly to $\cascadeucb$ and $\cascadeklucb$.

\begin{theorem}
\label{thm:lower bound} For any cascading bandit $B_\mathrm{LB}$, the regret of any consistent algorithm is bounded from below as:
\begin{align*}
  \liminf_{n \to \infty} \frac{R(n)}{\log n} \geq
  \frac{(L - K) \Delta (1 - p)^{K - 1}}{\kl{p - \Delta}{p}}\,.
\end{align*}
\end{theorem}
\begin{proof}
By \cref{thm:regret decomposition}, the expected regret at time $t$ conditioned on history $\cH_t$ is bounded from below as:
\begin{align*}
  \EEt{R(\rnd{A}_t, \rnd{w}_t)} \geq
  \Delta (1 - p)^{K - 1} \!\!\!\! \sum_{e = K + 1}^L \sum_{e^\ast = 1}^K \! \EE{\I{G_{e, e^\ast, t}}}\,.
\end{align*}
Based on this result, the $n$-step regret is bounded as:
\begin{align*}
  R(n)
  & \geq \Delta (1 - p)^{K - 1} \sum_{e = K + 1}^L
  \EE{\sum_{t = 1}^n \sum_{e^\ast = 1}^K \I{G_{e, e^\ast, t}}} \\
  & = \Delta (1 - p)^{K - 1} \sum_{e = K + 1}^L \EE{\rnd{T}_n(e)}\,,
\end{align*}
where the last step is based on the fact that the observation counter of item $e$ increases if and only if event $G_{e, e^\ast, t}$ happens. By the work of \citet{lai85asymptotically}, we have that for any suboptimal item $e$:
\begin{align*}
  \liminf_{n \to \infty} \frac{\EE{\rnd{T}_n(e)}}{\log n} \geq \frac{1}{\kl{p - \Delta}{p}}\,.
\end{align*}
Otherwise, the learning algorithm is unable to distinguish instances of our problem where item $e$ is optimal, and thus is not consistent. Finally, we chain all inequalities and get:
\begin{align*}
  \liminf_{n \to \infty} \frac{R(n)}{\log n} \geq
  \frac{(L - K) \Delta (1 - p)^{K - 1}}{\kl{p - \Delta}{p}}\,.
\end{align*}
This concludes our proof.
\end{proof}

Our lower bound is practical when no optimal item is very attractive, $p < 1 / K$. In this case, the learning agent must learn $K$ sufficiently attractive items to identify the optimal solution. This lower bound is not practical when $p$ is close to $1$, because it becomes exponentially small. In this case, other lower bounds would be more practical. For instance, consider a problem with $L$ items where item $1$ is attractive with probability one and all other items are attractive with probability zero. The optimal list of $K$ items in this problem can be found in $L / (2 K)$ steps in expectation.

%!TEX root = Paper.tex

\subsection{Discussion}
\label{sec:discussion}

We prove two gap-dependent upper bounds on the $n$-step regret of $\cascadeucb$ (\cref{thm:ucb1}) and $\cascadeklucb$ (\cref{thm:klucb}). The bounds are $O(\log n)$, linear in the number of items $L$, and they improve as the number of recommended items $K$ increases. The bounds do not depend on the order of recommended items. This is due to the nature of our proofs, where we count events that ignore the positions of the items. We would like to extend our analysis in this direction in future work.

We discuss the tightness of our upper bounds on problem $B_\mathrm{LB}(L, K, p, \Delta)$ in \cref{sec:lower bound} where we set $p = 1 / K$. In this problem, \cref{thm:lower bound} yields an asymptotic lower bound of:
\begin{align}
  \textstyle \Omega\left((L - K) \frac{\Delta}{\kl{p - \Delta}{p}} \log n\right)
  \label{eq:discussion lower bound}
\end{align}
since $(1 - 1 / K)^{K - 1} \geq 1 / e$ for $K > 1$. The $n$-step regret of $\cascadeucb$ is bounded by \cref{thm:ucb1} as:
\begin{align}
  & \textstyle O\left((L - K) \frac{1}{\Delta} \log n\right) \nonumber \\
  & \textstyle \ = O\left((L - K) \frac{\Delta}{\Delta^2} \log n\right) \nonumber \\
  & \textstyle \ = O\left((L - K) \frac{\Delta}{p (1 - p) \kl{p - \Delta}{p}} \log n\right) \nonumber \\
  & \textstyle \ = O\left(K (L - K) \frac{\Delta}{\kl{p - \Delta}{p}} \log n\right)\,,
  \label{eq:discussion ucb1}
\end{align}
where the second equality is by $\kl{p - \Delta}{p} \! \leq \! \frac{\Delta^2}{p (1 - p)}$. The $n$-step regret of $\cascadeklucb$ is bounded by \cref{thm:klucb} as:
\begin{align}
  \textstyle O\left((L - K) \frac{\Delta (1 + \log(1 / \Delta))}{\kl{p - \Delta}{p}} \log n\right)
  \label{eq:discussion klucb}
\end{align}
and matches the lower bound in \eqref{eq:discussion lower bound} up to $\log(1 / \Delta)$. Note that the upper bound of $\cascadeklucb$ \eqref{eq:discussion klucb} is below that of $\cascadeucb$ \eqref{eq:discussion ucb1} when $\log(1 / \Delta) = O(K)$, or equivalently when $\Delta = \Omega(e^{- K})$. It is an open problem whether the factor of $\log(1 / \Delta)$ in \eqref{eq:discussion klucb} can be eliminated.

%!TEX root = Paper.tex

\section{Experiments}
\label{sec:experiments}

\begin{table}[t]
  \centering
  {\small
  \begin{tabular}{rrr|rr} \hline
    $L$ & $K$ & $\Delta$ & $\cascadeucb$ & $\cascadeklucb$ \\ \hline
    16 & 2 & 0.15 & $1290.1 \pm 11.3$ & $357.9 \pm 5.5\phantom{0}$ \\
    16 & 4 & 0.15 & $986.8 \pm 10.8$ & $275.1 \pm 5.8\phantom{0}$ \\
    16 & 8 & 0.15 & $574.8 \pm 7.9\phantom{0}$ & $149.1 \pm 3.2\phantom{0}$ \\
    32 & 2 & 0.15 & $2695.9 \pm 19.8$ & $761.2 \pm 10.4$ \\
    32 & 4 & 0.15 & $2256.8 \pm 12.8$ & $633.2 \pm 7.0\phantom{0}$ \\
    32 & 8 & 0.15 & $1581.0 \pm 20.3$ & $435.4 \pm 5.7\phantom{0}$ \\
    16 & 2 & 0.075 & $2077.0 \pm 32.9$ & $766.0 \pm 18.0$ \\
    16 & 4 & 0.075 & $1520.4 \pm 23.4$ & $538.5 \pm 12.5$ \\
    16 & 8 & 0.075 & $725.4 \pm 12.0$ & $321.0 \pm 16.3$ \\ \hline
  \end{tabular}
  }
  \caption{The $n$-step regret of $\cascadeucb$ and $\cascadeklucb$ in $n = 10^5$ steps. The list $\rnd{A}_t$ is ordered from the largest UCB to the smallest. All results are averaged over $20$ runs.}
  \label{tab:regret bounds}
\end{table}

We conduct four experiments. In \cref{sec:experiments regret bounds}, we validate that the regret of our algorithms scales as suggested by our upper bounds (\cref{sec:upper bounds}). In \cref{sec:experiments worst-first}, we experiment with recommending items $\rnd{A}_t$ in the opposite order, in increasing order of their UCBs. In \cref{sec:experiments imperfect model}, we show that $\cascadeklucb$ performs robustly even when our modeling assumptions are violated. In \cref{sec:experiments ranked bandits}, we compare $\cascadeklucb$ to ranked bandits.

\subsection{Regret Bounds}
\label{sec:experiments regret bounds}

In the first experiment, we validate the qualitative behavior of our upper bounds (\cref{sec:upper bounds}). We experiment with the class of problems $B_\mathrm{LB}(L, K, p, \Delta)$ in \cref{sec:lower bound}. We set $p = 0.2$; and vary $L$, $K$, and $\Delta$. The attraction probability $p$ is set such that it is close to $1 / K$ for the maximum value of $K$ in our experiments. Our upper bounds are reasonably tight in this setting (\cref{sec:discussion}), and we expect the regret of our methods to scale accordingly. We recommend items $\rnd{A}_t$ in decreasing order of their UCBs. This order is motivated by the problem of web search, where higher ranked items are typically more attractive. We run $\cascadeucb$ and $\cascadeklucb$ for $n = 10^5$ steps.

Our results are reported in \cref{tab:regret bounds}. We observe four major trends. First, the regret doubles when the number of items $L$ doubles. Second, the regret decreases when the number of recommended items $K$ increases. These trends are consistent with the fact that our upper bounds are $O(L - K)$. Third, the regret increases when $\Delta$ decreases. Finally, note that $\cascadeklucb$ outperforms $\cascadeucb$. This result is not particularly surprising. $\klucb$ is known to outperform $\ucb$ when the expected payoffs of arms are low \cite{garivier11klucb}, because its confidence intervals get tighter as the Bernoulli parameters get closer to $0$ or $1$.

\subsection{Worst-of-Best First Item Ordering}
\label{sec:experiments worst-first}

\begin{table}[t]
  \centering
  {\small
  \begin{tabular}{rrr|rr} \hline
    $L$ & $K$ & $\Delta$ & $\cascadeucb$ & $\cascadeklucb$ \\ \hline
    16 & 2 & 0.15 & $1160.2 \pm 11.7$ & $333.3 \pm 6.1\phantom{0}$ \\
    16 & 4 & 0.15 & $660.0 \pm 8.3\phantom{0}$ & $209.4 \pm 4.4\phantom{0}$ \\
    16 & 8 & 0.15 & $181.4 \pm 3.9\phantom{0}$ & $60.4 \pm 2.0\phantom{0}$ \\
    32 & 2 & 0.15 & $2471.6 \pm 14.1$ & $716.0 \pm 7.5\phantom{0}$ \\
    32 & 4 & 0.15 & $1615.3 \pm 14.5$ & $482.3 \pm 6.7\phantom{0}$ \\
    32 & 8 & 0.15 & $595.0 \pm 7.8\phantom{0}$ & $201.9 \pm 5.8\phantom{0}$ \\
    16 & 2 & 0.075 & $1989.8 \pm 31.4$ & $785.8 \pm 12.2$ \\
    16 & 4 & 0.075 & $1239.5 \pm 16.2$ & $484.2 \pm 12.5$ \\
    16 & 8 & 0.075 & $336.4 \pm 10.3$ & $139.7 \pm 6.6\phantom{0}$ \\ \hline
  \end{tabular}
  }
  \caption{The $n$-step regret of $\cascadeucb$ and $\cascadeklucb$ in $n = 10^5$ steps. The list $\rnd{A}_t$ is ordered from the smallest UCB to the largest. All results are averaged over $20$ runs.}
  \label{tab:worst-first}
\end{table}

In the second experiment, we recommend items $\rnd{A}_t$ in increasing order of their UCBs. This choice is not very natural and may be even dangerous. In practice, the user could get annoyed if highly ranked items were not attractive. On the other hand, the user would provide a lot of feedback on low quality items, which could speed up learning. We note that the reward in our model does not depend on the order of recommended items (\cref{sec:algorithms}). Therefore, the items can be ordered arbitrarily, perhaps to maximize feedback. In any case, we find it important to study the effect of this counterintuitive ordering, at least to demonstrate the effect of our modeling assumptions.

The experimental setup is the same as in \cref{sec:experiments regret bounds}. Our results are reported in \cref{tab:worst-first}. When compared to \cref{tab:regret bounds}, the regret of $\cascadeucb$ and $\cascadeklucb$ decreases for all settings of $K$, $L$, and $\Delta$; most prominently at large values of $K$. Our current analysis cannot explain this phenomenon and we leave it for future work.

\subsection{Imperfect Model}
\label{sec:experiments imperfect model}

The goal of this experiment is to evaluate $\cascadeklucb$ in the setting where our modeling assumptions are not satisfied, to test its potential beyond our model. We generate data from the \emph{dynamic Bayesian network (DBN) model} of \citet{chapelle09dynamic}, a popular extension of the cascade model which is parameterized by \emph{attraction probabilities} $\rho \in [0, 1]^E$, \emph{satisfaction probabilities} $\nu \in [0, 1]^E$, and the \emph{persistence} of users $\gamma \in (0, 1]$. In the DBN model, the user is recommended a list of $K$ items $A = (a_1, \dots, a_K)$ and examines it from the first recommended item $a_1$ to the last $a_K$. After the user examines item $a_k$, the item attracts the user with probability $\rho(a_k)$. When the user is attracted by the item, the user clicks on it and is satisfied with probability $\nu(a_k)$. If the user is satisfied, the user does not examine the remaining items. In any other case, the user examines item $a_{k + 1}$ with probability $\gamma$. The reward is one if the user is satisfied with the list, and zero otherwise. Note that this is not observed. The regret is defined accordingly. The feedback are clicks of the user. Note that the user can click on multiple items.

The probability that at least one item in $A = (a_1, \dots, a_K)$ is satisfactory is:
\begin{align*}
  \sum_{k = 1}^K \gamma^{k - 1} \bar{w}(a_k) \prod_{i = 1}^{k - 1} (1 - \bar{w}(a_i))\,,
\end{align*}
where $\bar{w}(e) = \rho(e) \nu(e)$ is the probability that item $e$ satisfies the user after being examined. This objective is maximized by the list of $K$ items with largest weights $\bar{w}(e)$ that are ordered in decreasing order of their weights. Note that the order matters.

The above objective is similar to that in cascading bandits (\cref{sec:cascading bandits}). Therefore, it may seem that our learning algorithms (\cref{sec:algorithms}) can also learn the optimal solution to the DBN model. Unfortunately, this is not guaranteed. The reason is that not all clicks of the user are satisfactory. We illustrate this issue on a simple problem. Suppose that the user clicks on multiple items. Then only the last click \emph{can} be satisfactory. But it \emph{does not have to} be. For instance, it could have happened that the user was unsatisfied with the last click, and then scanned the recommended list until the end and left.

We experiment on the class of problems $B_\mathrm{LB}(L, K, p, \Delta)$ in \cref{sec:lower bound} and modify it as follows. The ground set $E$ has $L = 16$ items and $K = 4$. The attraction probability of item $e$ is $\rho(e) = \bar{w}(e)$, where $\bar{w}(e)$ is given in \eqref{eq:attraction probability}. We set $\Delta = 0.15$. The satisfaction probabilities $\nu(e)$ of all items are the same. We experiment with two settings of $\nu(e)$, $1$ and $0.7$; and with two settings of persistence $\gamma$, $1$ and $0.7$. We run $\cascadeklucb$ for $n = 10^5$ steps and use the last click as an indicator that the user is satisfied with the item.

Our results are reported in Figure~\ref{fig:DBN trends}. We observe in all experiments that the regret of $\cascadeklucb$ flattens. This indicates that $\cascadeklucb$ learns the optimal solution to the DBN model. An intuitive explanation for this result is that the exact values of $\bar{w}(e)$ are not needed to perform well. Our current theory does not explain this phenomenon and we leave it for future work.

\subsection{Ranked Bandits}
\label{sec:experiments ranked bandits}

\begin{figure}[t]
  \centering
  \includegraphics[width=3.2in, bb=2.15in 4.25in 6.15in 6.75in]{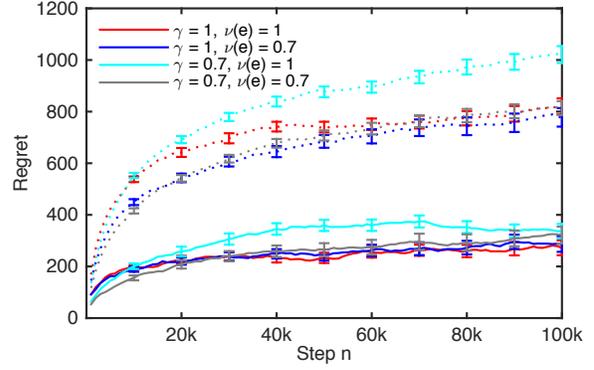}
  \caption{The $n$-step regret of $\cascadeklucb$ (solid lines) and $\rankedklucb$ (dotted lines) in the DBN model in \cref{sec:experiments imperfect model}.}
  \label{fig:DBN trends}
\end{figure}

In our final experiment, we compare $\cascadeklucb$ to a ranked bandit (\cref{sec:related work}) where the base bandit algorithm is $\klucb$. We refer to this method as $\rankedklucb$. The choice of the base algorithm is motivated by the following reasons. First, $\klucb$ is the best performing oracle in our experiments. Second, since both compared approaches use the same oracle, the difference in their regrets is likely due to their statistical efficiency, and not the oracle itself.

The experimental setup is the same as in \cref{sec:experiments imperfect model}. Our results are reported in Figure~\ref{fig:DBN trends}. We observe that the regret of $\rankedklucb$ is significantly larger than the regret of $\cascadeklucb$, about three times. The reason is that the regret in ranked bandits is $\Omega(K)$ (\cref{sec:related work}) and $K = 4$ in this experiment. The regret of our algorithms is $O(L - K)$ (\cref{sec:discussion}). Note that $\cascadeklucb$ is not guaranteed to be optimal in this experiment. Therefore, our results are encouraging and show that $\cascadeklucb$ could be a viable alternative to more established approaches.

%!TEX root = Paper.tex

\section{Related Work}
\label{sec:related work}

\emph{Ranked bandits} are a popular approach in learning to rank \cite{radlinski08learning} and they are closely related to our work. The key characteristic of ranked bandits is that each position in the recommended list is an independent bandit problem, which is solved by some \emph{base bandit algorithm}. The solutions in ranked bandits are $(1 - 1 / e)$ approximate and the regret is $\Omega(K)$ \cite{radlinski08learning}, where $K$ is the number of recommended items. Cascading bandits can be viewed as a form of ranked bandits where each recommended item attracts the user independently. We propose novel algorithms for this setting that can learn the optimal solution and whose regret decreases with $K$. We compare one of our algorithms to ranked bandits in \cref{sec:experiments ranked bandits}.

Our learning problem is of a combinatorial nature, our objective is to learn $K$ most attractive items out of $L$. In this sense, our work is related to stochastic combinatorial bandits, which are often studied with linear rewards and semi-bandit feedback \cite{gai12combinatorial,kveton14matroid,kveton14learning,kveton15tight}. The key differences in our work are that the reward function is non-linear in unknown parameters; and that the feedback is less than semi-bandit, only a subset of the recommended items is observed.

Our reward function is non-linear in unknown parameters. These types of problems have been studied before in various contexts. \citet{filippi10parametric} proposed and analyzed a generalized linear bandit with bandit feedback. \citet{chen13combinatorial} studied a variant of stochastic combinatorial semi-bandits whose reward function is a known monotone function of a linear function in unknown parameters. \citet{le14sequential} studied a network optimization problem whose reward function is a non-linear function of observations.

\citet{bartok12adaptive} studied finite partial monitoring problems. This is a very general class of problems with finitely many actions, which are chosen by the learning agent; and finitely many outcomes, which are determined by the environment. The outcome is unobserved and must be inferred from the feedback of the environment. Cascading bandits can be viewed as finite partial monitoring problems where the actions are lists of $K$ items out of $L$ and the outcomes are the corners of a $L$-dimensional binary hypercube. \citet{bartok12adaptive} proposed an algorithm that can solve such problems. This algorithm is computationally inefficient in our problem because it needs to reason over all pairs of actions and stores vectors of length $2^L$. \citet{bartok12adaptive} also do not prove logarithmic distribution-dependent regret bounds as in our work.

\citet{agrawal89asymptotically} studied a partial monitoring problem with non-linear rewards. In this problem, the environment draws a state from a distribution that depends on the action of the learning agent and an unknown parameter. The form of this dependency is known. The state of the environment is observed and determines reward. The reward is a known function of the state and action. \citet{agrawal89asymptotically} also proposed an algorithm for their problem and proved a logarithmic distribution-dependent regret bound. Similarly to \citet{bartok12adaptive}, this algorithm is computationally inefficient in our setting.

\citet{lin14combinatorial} studied partial monitoring in combinatorial bandits. The setting of this work is different from ours. \citet{lin14combinatorial} assume that the feedback is a linear function of the weights of the items that is indexed by actions. Our feedback is a non-linear function of the weights of the items.

\citet{mannor11bandits} and \citet{caron12leveraging} studied an opposite setting to ours, where the learning agent observes a superset of chosen items. \citet{chen14combinatorial} studied this problem in stochastic combinatorial semi-bandits.

%!TEX root = Paper.tex

\section{Conclusions}
\label{sec:conclusions}

In this paper, we propose a learning variant of the cascade model \cite{craswell08experimental}, a popular model of user behavior in web search. We propose two algorithms for solving it, $\cascadeucb$ and $\cascadeklucb$, and prove gap-dependent upper bounds on their regret. Our analysis addresses two main challenges of our problem, a non-linear reward function and limited feedback. We evaluate our algorithms on several problems and show that they perform well even when our modeling assumptions are violated.

We leave open several questions of interest. For instance, we show in \cref{sec:experiments imperfect model} that  $\cascadeklucb$ can learn the optimal solution to the DBN model. This indicates that the DBN model is learnable in the bandit setting and we leave this for future work. Note that the regret in cascading bandits is $\Omega(L)$ (\cref{sec:lower bound}). Therefore, our learning framework is not practical when the number of items $L$ is large. Similarly to \citet{slivkins13ranked}, we plan to address this issue by embedding the items in some feature space, along the lines of  \citet{wen15efficient}. Finally, we want to generalize our results to more complex problems, such as learning routing paths in computer networks where the connections fail with unknown probabilities.

From the theoretical point of view, we would like to close the gap between our upper and lower bounds. In addition, we want to derive gap-free bounds. Finally, we would like to refine our analysis so that it explains that the reverse ordering of recommended items yields smaller regret.

\bibliographystyle{icml2015}
\bibliography{References}

%!TEX root = Paper.tex

\clearpage
\onecolumn
\appendix

\section{Proofs of Main Theorems}
\label{sec:proofs}

\subsection{Proof of \cref{thm:ucb1}}
\label{sec:proof ucb1}

Let $\rnd{R}_t = R(\rnd{A}_t, \rnd{w}_t)$ be the regret of the learning algorithm at time $t$, where $\rnd{A}_t$ is the recommended list at time $t$ and $\rnd{w}_t$ are the weights of items at time $t$. Let $\cE_t = \set{\exists e \in E \text{ s.t. } \abs{\bar{w}(e) - \hat{\rnd{w}}_{\rnd{T}_{t - 1}(e)}(e)} \geq c_{t - 1, \rnd{T}_{t - 1}(e)}}$ be the event that $\bar{w}(e)$ is not in the high-probability confidence interval around $\hat{\rnd{w}}_{\rnd{T}_{t - 1}(e)}(e)$ for some $e$ at time $t$; and let $\ccE_t$ be the complement of $\cE_t$, $\bar{w}(e)$ is in the high-probability confidence interval around $\hat{\rnd{w}}_{\rnd{T}_{t - 1}(e)}(e)$ for all $e$ at time $t$. Then we can decompose the regret of $\cascadeucb$ as:
\begin{align}
  R(n) =
  \EE{\sum_{t = 1}^n \I{\cE_t} \rnd{R}_t} +
  \EE{\sum_{t = 1}^n \I{\ccE_t} \rnd{R}_t}\,.
  \label{eq:good bad ucb1}
\end{align}
Now we bound both terms in the above regret decomposition.

The first term in \eqref{eq:good bad ucb1} is small because all of our confidence intervals hold with high probability.
In particular, Hoeffding's inequality \citep[Theorem 2.8]{boucheron13concentration} yields that for any $e$, $s$, and $t$:
\begin{align*}
  P(\abs{\bar{w}(e) - \hat{\rnd{w}}_s(e)} \geq c_{t, s}) \leq 2 \exp[-3 \log t]\,,
\end{align*}
and therefore:
\begin{align*}
  \EE{\sum_{t = 1}^n \I{\cE_t}} \leq
  \sum_{e \in E} \sum_{t = 1}^n \sum_{s = 1}^t P(\abs{\bar{w}(e) - \hat{\rnd{w}}_s(e)} \geq c_{t, s}) \leq
  2 \sum_{e \in E} \sum_{t = 1}^n \sum_{s = 1}^t \exp[-3 \log t] \leq
  2 \sum_{e \in E} \sum_{t = 1}^n t^{-2} \leq
  \frac{\pi^2}{3} L\,.
\end{align*}
Since $\rnd{R}_t \leq 1$, $\EE{\sum_{t = 1}^n \I{\cE_t} \rnd{R}_t} \leq \frac{\pi^2}{3} L$.

Recall that $\EEt{\cdot} = \condEE{\cdot}{\cH_t}$, where $\cH_t$ is the history of the learning agent up to choosing $\rnd{A}_t$, the first $t - 1$ observations and $t$ actions \eqref{eq:history}. Based on this definition, we rewrite the second term in \eqref{eq:good bad ucb1} as:
\begin{align*}
  \EE{\sum_{t = 1}^n \I{\ccE_t} \rnd{R}_t} \stackrel{\text{(a)}}{=}
  \sum_{t = 1}^n \EE{\I{\ccE_t} \EEt{\rnd{R}_t}} \stackrel{\text{(b)}}{\leq}
  \sum_{e = K + 1}^L \EE{\sum_{e^\ast = 1}^K \sum_{t = 1}^n \Delta_{e, e^\ast} \I{\ccE_t, G_{e, e^\ast, t}}}\,,
\end{align*}
where equality (a) is due to the tower rule and that $\I{\ccE_t}$ is only a function of $\cH_t$, and inequality (b) is due to the upper bound in \cref{thm:regret decomposition}.

Now we bound $\sum_{e^\ast = 1}^K \sum_{t = 1}^n \Delta_{e, e^\ast} \I{\ccE_t, G_{e, e^\ast, t}}$ for any suboptimal item $e$. Select any optimal item $e^\ast$. When event $\ccE_t$ happens, $\abs{\bar{w}(e) - \hat{\rnd{w}}_{\rnd{T}_{t - 1}(e)}(e)} < c_{t - 1, \rnd{T}_{t - 1}(e)}$. Moreover, when event $G_{e, e^\ast, t}$ happens, $\rnd{U}_t(e) \geq \rnd{U}_t(e^\ast)$ by \cref{thm:regret decomposition}. Therefore, when both $G_{e, e^\ast, t}$ and $\ccE_t$ happen:
\begin{align*}
  \bar{w}(e) + 2 c_{t - 1, \rnd{T}_{t - 1}(e)} \geq
  \rnd{U}_t(e) \geq
  \rnd{U}_t(e^\ast) \geq
  \bar{w}(e^\ast)\,,
\end{align*}
which implies:
\begin{align*}
  2 c_{t - 1, \rnd{T}_{t - 1}(e)} \geq \Delta_{e, e^\ast}\,.
\end{align*}
Together with $c_{n, \rnd{T}_{t - 1}(e)} \geq c_{t - 1, \rnd{T}_{t - 1}(e)}$, this implies $\rnd{T}_{t - 1}(e) \leq \tau_{e, e^\ast}$, where $\tau_{e, e^\ast} = \frac{6}{\Delta_{e, e^\ast}^2} \log n$. Therefore:
\begin{align}
  \sum_{e^\ast = 1}^K \sum_{t = 1}^n \Delta_{e, e^\ast} \I{\ccE_t, G_{e, e^\ast, t}} \leq
  \sum_{e^\ast = 1}^K \Delta_{e, e^\ast} \sum_{t = 1}^n
  \I{\rnd{T}_{t - 1}(e) \leq \tau_{e, e^\ast}, \ G_{e, e^\ast, t}}\,.
  \label{eq:peel me}
\end{align}
Let:
\begin{align*}
  \rnd{M}_{e, e^\ast} = \sum_{t = 1}^n \I{\rnd{T}_{t - 1}(e) \leq \tau_{e, e^\ast}, \ G_{e, e^\ast, t}}
\end{align*}
be the inner sum in \eqref{eq:peel me}. Now note that (i) the counter $\rnd{T}_{t - 1}(e)$ of item $e$ increases by one when the event $G_{e, e^\ast, t}$ happens for any optimal item $e^\ast$, (ii) the event $G_{e, e^\ast, t}$ happens for at most one optimal $e^\ast$ at any time $t$; and (iii) $\tau_{e, 1} \leq \ldots \leq \tau_{e, K}$. Based on these facts, it follows that $\rnd{M}_{e, e^\ast} \leq \tau_{e, e^\ast}$, and moreover $\sum_{e^\ast = 1}^K \rnd{M}_{e, e^\ast} \leq \tau_{e, K}$. Therefore, the right-hand side of \eqref{eq:peel me} can be bounded from above by:
\begin{align*}
  \max \left\{\sum_{e^\ast = 1}^K \Delta_{e, e^\ast} m_{e, e^\ast}:
  0 \leq m_{e, e^\ast} \leq \tau_{e, e^\ast}, \ \sum_{e^\ast = 1}^K m_{e, e^\ast} \leq \tau_{e, K}\right\}\,.
\end{align*}
Since the gaps are decreasing, $\Delta_{e, 1} \geq \ldots \geq \Delta_{e, K}$, the solution to the above problem is $m_{e, 1}^\ast = \tau_{e, 1}$, $m_{e, 2}^\ast = \tau_{e, 2} - \tau_{e, 1}$, $\dots$, $m_{e, K}^\ast = \tau_{e, K} - \tau_{e, K - 1}$. Therefore, the value of \eqref{eq:peel me} is bounded from above by:
\begin{align*}
  \left[\Delta_{e, 1} \frac{1}{\Delta_{e, 1}^2} + \sum_{e^\ast = 2}^K \Delta_{e, e^\ast}
  \left(\frac{1}{\Delta_{e, e^\ast}^2} - \frac{1}{\Delta_{e, e^\ast - 1}^2}\right)\right] 6 \log n\,.
\end{align*}
By Lemma 3 of \citet{kveton14matroid}, the above term is bounded by $\frac{12}{\Delta_{e, K}} \log n$. Finally, we chain all inequalities and sum over all suboptimal items $e$.

\subsection{Proof of \cref{thm:klucb}}
\label{sec:proof klucb}

Let $\rnd{R}_t = R(\rnd{A}_t, \rnd{w}_t)$ be the regret of the learning algorithm at time $t$, where $\rnd{A}_t$ is the recommended list at time $t$ and $\rnd{w}_t$ are the weights of items at time $t$. Let $\cE_t = \set{\exists 1 \leq e \leq K \text{ s.t. } \bar{w}(e) > \rnd{U}_t(e)}$ be the event that the attraction probability of at least one optimal item is above its upper confidence bound at time $t$. Let $\ccE_t$ be the complement of event $\cE_t$. Then we can decompose the regret of $\cascadeklucb$ as:
\begin{align}
  R(n) =
  \EE{\sum_{t = 1}^n \I{\cE_t} \rnd{R}_t} +
  \EE{\sum_{t = 1}^n \I{\ccE_t} \rnd{R}_t}\,.
  \label{eq:good bad klucb}
\end{align}
By Theorems 2 and 10 of \citet{garivier11klucb}, thanks to the choice of the upper confidence bound $\rnd{U}_t$, the first term in \eqref{eq:good bad klucb} is bounded as $\EE{\sum_{t = 1}^n \I{\cE_t} \rnd{R}_t} \leq 7 K \log \log n$. As in the proof of \cref{thm:ucb1}, we rewrite the second term as:
\begin{align*}
  \EE{\sum_{t = 1}^n \I{\ccE_t} \rnd{R}_t} =
  \sum_{t = 1}^n \EE{\I{\ccE_t} \EEt{\rnd{R}_t}} \leq
  \sum_{e = K + 1}^L \EE{\sum_{e^\ast = 1}^K \sum_{t = 1}^n \Delta_{e, e^\ast} \I{\ccE_t, G_{e, e^\ast, t}}}\,.
\end{align*}
Now note that for any suboptimal item $e$ and $\tau_{e, e^\ast} > 0$:
\begin{align}
  \EE{\sum_{e^\ast = 1}^K \sum_{t = 1}^n \Delta_{e, e^\ast} \I{\ccE_t, G_{e, e^\ast, t}}}
  & \leq \EE{\sum_{e^\ast = 1}^K \sum_{t = 1}^n
  \Delta_{e, e^\ast} \I{\rnd{T}_{t - 1}(e) \leq \tau_{e, e^\ast}, \ G_{e, e^\ast, t}}} + {} \label{eq:pull split} \\
  & \phantom{{} = {}} \sum_{e^\ast = 1}^K \Delta_{e, e^\ast}
  \EE{\sum_{t = 1}^n \I{\rnd{T}_{t - 1}(e) > \tau_{e, e^\ast}, \ \ccE_t, \ G_{e, e^\ast, t}}}\,. \nonumber
\end{align}
Let:
\begin{align*}
  \tau_{e, e^\ast} = \frac{1 + \eps}{\kl{\bar{w}(e)}{\bar{w}(e^\ast)}}(\log n + 3 \log \log n)\,.
\end{align*}
Then by the same argument as in Theorem 2 and Lemma 8 of \citet{garivier11klucb}:
\begin{align*}
  \EE{\sum_{t = 1}^n \I{\rnd{T}_{t - 1}(e) > \tau_{e, e^\ast}, \ \ccE_t, \ G_{e, e^\ast, t}}} \leq
  \frac{C_2(\eps)}{n^{\beta(\eps)}}
\end{align*}
holds for any suboptimal $e$ and optimal $e^\ast$. So the second term in \eqref{eq:pull split} is bounded from above by $K \frac{C_2(\eps)}{n^{\beta(\eps)}}$. Now we bound the first term in \eqref{eq:pull split}. By the same argument as in the proof of \cref{thm:ucb1}:
\begin{align*}
  & \sum_{e^\ast = 1}^K \sum_{t = 1}^n
  \Delta_{e, e^\ast} \I{\rnd{T}_{t - 1}(e) \leq \tau_{e, e^\ast}, \ G_{e, e^\ast, t}} \leq \\
  & \quad \left[\frac{\Delta_{e, 1}}{\kl{\bar{w}(e)}{\bar{w}(1)}} + \sum_{e^\ast = 2}^K \Delta_{e, e^\ast}
  \left(\frac{1}{\kl{\bar{w}(e)}{\bar{w}(e^\ast)}} - \frac{1}{\kl{\bar{w}(e)}{\bar{w}(e^\ast - 1)}}\right)\right]
  (1 + \eps) (\log n + 3 \log \log n)
\end{align*}
holds for any suboptimal item $e$. By \cref{lem:klucb peeling}, the leading constant is bounded as:
\begin{align*}
  \frac{\Delta_{e, 1}}{\kl{\bar{w}(e)}{\bar{w}(1)}} + \sum_{e^\ast = 2}^K \Delta_{e, e^\ast}
  \left(\frac{1}{\kl{\bar{w}(e)}{\bar{w}(e^\ast)}} - \frac{1}{\kl{\bar{w}(e)}{\bar{w}(e^\ast - 1)}}\right) \leq
  \frac{\Delta_{e, K} (1 + \log(1 / \Delta_{e, K}))}{\kl{\bar{w}(e)}{\bar{w}(K)}}\,.
\end{align*}
Finally, we chain all inequalities and sum over all suboptimal items $e$.

\section{Technical Lemmas}
\label{sec:lemmas}

\newtheorem*{lem:modular decomposition}{\cref{lem:modular decomposition}}
\begin{lem:modular decomposition}
Let $A = (a_1, \dots, a_K)$ and $B = (b_1, \dots, b_K)$ be any two lists of $K$ items from $\Pi_K(E)$ such that $a_i = b_j$ only if $i = j$. Let $\rnd{w} \sim P$ in \cref{ass:independence}. Then:
\begin{align*}
  \EE{\prod_{k = 1}^K \rnd{w}(a_k) - \prod_{k = 1}^K \rnd{w}(b_k)} =
  \sum_{k = 1}^K \EE{\prod_{i = 1}^{k - 1} \rnd{w}(a_i)}
  \EE{\rnd{w}(a_k) - \rnd{w}(b_k)} \left(\prod_{j = k + 1}^K \EE{\rnd{w}(b_j)}\right)\,.
\end{align*}
\end{lem:modular decomposition}
\begin{proof}
First, we prove that:
\begin{align*}
  \prod_{k = 1}^K w(a_k) - \prod_{k = 1}^K w(b_k) =
  \sum_{k = 1}^K \left(\prod_{i = 1}^{k - 1} w(a_i)\right) (w(a_k) - w(b_k)) \left(\prod_{j = k + 1}^K w(b_j)\right)
\end{align*}
holds for any $w \in \set{0, 1}^L$. The proof is by induction on $K$. The claim holds obviously for $K = 1$. Now suppose that the claim holds for any $A, B \in \Pi_{K - 1}(E)$. Let $A, B \in \Pi_K(E)$. Then:
\begin{align*}
  \prod_{k = 1}^K w(a_k) - \prod_{k = 1}^K w(b_k)
  & = \prod_{k = 1}^K w(a_k) - w(b_K) \prod_{k = 1}^{K - 1} w(a_k) +
  w(b_K) \prod_{k = 1}^{K - 1} w(a_k) - \prod_{k = 1}^K w(b_k) \\
  & = (w(a_K) - w(b_K)) \prod_{k = 1}^{K - 1} w(a_k) +
  w(b_K) \left[\prod_{k = 1}^{K - 1} w(a_k) - \prod_{k = 1}^{K - 1} w(b_k)\right] \\
  & = (w(a_K) - w(b_K)) \prod_{k = 1}^{K - 1} w(a_k) +
  \sum_{k = 1}^{K - 1} \left(\prod_{i = 1}^{k - 1} w(a_i)\right)
  (w(a_k) - w(b_k)) \left(\prod_{j = k + 1}^K w(b_j)\right) \\
  & = \sum_{k = 1}^K \left(\prod_{i = 1}^{k - 1} w(a_i)\right)
  (w(a_k) - w(b_k)) \left(\prod_{j = k + 1}^K w(b_j)\right)\,.
\end{align*}
The third equality is by our induction hypothesis. Finally, note that $\rnd{w}$ is drawn from a factored distribution. Therefore, we can decompose the expectation of the product as a product of expectations, and our claim follows.
\end{proof}

\begin{lemma}
\label{lem:klucb peeling} Let $p_1 \geq \ldots \geq p_K > p$ be $K + 1$ probabilities and $\Delta_k = p_k - p$ for $1 \leq k \leq K$. Then:
\begin{align*}
  \frac{\Delta_1}{\kl{p}{p_1}} + \sum_{k = 2}^K \Delta_k
  \left(\frac{1}{\kl{p}{p_k}} - \frac{1}{\kl{p}{p_{k - 1}}}\right) \leq
  \frac{\Delta_K (1 + \log(1 / \Delta_K))}{\kl{p}{p_K}}\,.
\end{align*}
\end{lemma}
\begin{proof}
First, we note that:
\begin{align*}
  \frac{\Delta_1}{\kl{p}{p_1}} + \sum_{k = 2}^K \Delta_k
  \left(\frac{1}{\kl{p}{p_k}} - \frac{1}{\kl{p}{p_{k - 1}}}\right) =
  \sum_{k = 1}^{K - 1} \frac{\Delta_k - \Delta_{k + 1}}{\kl{p}{p_k}} + \frac{\Delta_K}{\kl{p}{p_K}}\,.
\end{align*}
The summation over $k$ can be bounded from above by a definite integral:
\begin{align*}
  \sum_{k = 1}^{K - 1} \frac{\Delta_k - \Delta_{k + 1}}{\kl{p}{p_k}} =
  \sum_{k = 1}^{K - 1} \frac{\Delta_k - \Delta_{k + 1}}{\kl{p}{p + \Delta_k}} \leq
  \int_{\Delta_K}^{\Delta_1} \frac{1}{\kl{p}{p + x}} \ud x \leq
  \int_{\Delta_K}^1 \frac{1}{\kl{p}{p + x}} \ud x\,,
\end{align*}
where the first inequality follows from the fact that $1 / \kl{p}{p + x}$ decreases on $x \geq 0$. To the best of our knowledge, the integral of $1 / \kl{p}{p + x}$ over $x$ does not have a simple analytic solution. Therefore, we integrate an upper bound on $1 / \kl{p}{p + x}$ which does. In particular, note that for any $x \geq \Delta_K$:
\begin{align*}
  \kl{p}{p + x} \geq \frac{\kl{p}{p + \Delta_K}}{\Delta_K} x = \frac{\kl{p}{p_K}}{\Delta_K} x
\end{align*}
because $\kl{p}{p + x}$ is convex, increasing in $x \geq 0$, and its minimum is attained at $x = 0$. Therefore:
\begin{align*}
  \int_{\Delta_K}^1 \frac{1}{\kl{p}{p + x}} \ud x \leq
  \frac{\Delta_K}{\kl{p}{p_K}} \int_{\Delta_K}^1 \frac{1}{x} \ud x =
  \frac{\Delta_K}{\kl{p}{p_K}} \log(1 / \Delta_K)\,.
\end{align*}
Finally, we chain all inequalities and get the final result.
\end{proof}

\end{document}